%% file: main.tex
\documentclass[12pt]{article}

\usepackage{fullpage}
\usepackage[ruled,vlined,algo2e,linesnumbered]{algorithm2e}
\usepackage[utf8]{inputenc} 
\usepackage[T1]{fontenc}    
\usepackage{hyperref}  
\usepackage{amsfonts}      
\usepackage{graphicx}
\usepackage{amsmath}
\usepackage{amsthm}
\usepackage{amssymb}
\usepackage{natbib}
\usepackage{tikz}
\DeclareMathOperator*{\argmin}{arg\,min}

\newtheorem{theorem}{Theorem}
\newtheorem{definition}{Definition}

\title{Labeled Optimal Partitioning}

\author{
  Toby Dylan Hocking, 
  \texttt{toby.hocking@nau.edu} \\
  Anuraag Srivastava,
  \texttt{as4378@nau.edu}
}

\begin{document}

\maketitle

\begin{abstract}
  In data sequences measured over space or time, an important problem is accurate detection of abrupt changes. 
  In partially labeled data, it is important to correctly predict presence/absence of changes in positive/negative labeled regions, in both the train and test sets.
  One existing dynamic programming algorithm is designed for prediction in unlabeled test regions (and ignores the labels in the train set); another is for accurate fitting of train labels (but does not predict changepoints in unlabeled test regions).
  We resolve these issues by proposing a new optimal changepoint detection model that is guaranteed to fit the labels in the train data, and can also provide predictions of unlabeled changepoints in test data.
  We propose a new dynamic programming algorithm, Labeled Optimal Partitioning (LOPART), and we provide a formal proof that it solves the resulting non-convex optimization problem.
  We provide theoretical and empirical analysis of the time complexity of our algorithm, in terms of the number of labels and the size of the data sequence to segment.
  Finally, we provide empirical evidence that our algorithm is more accurate than the existing baselines, in terms of train and test label error.
\end{abstract}

\section{Introduction}

In the context of fields such as medical monitoring \citep{Fotoohinasab2020} and genomics \citep{HOCKING-penalties}, where data are measured over space or time, detecting abrupt changes is an important problem.
There are many different algorithms available for detecting changepoints, and in this paper we focus on algorithms that compute a solution to a well-defined mathematical optimization problem. 
For example, the classical optimal partitioning (OPART) algorithm was introduced by \citet{Jackson2005} in order to solve the penalized changepoint problem for a sequence of $N$ data $\mathbf x = [ x_1 \cdots x_N ]$. 
Inferring the most likely parameter vector $\mathbf m$ corresponds to minimizing a sum of losses $\ell$ plus a penalty $\lambda$ for each changepoint:
\begin{equation}
    \label{eq:op}
    \hat C_N = \min_{\mathbf m\in \mathbb R^N}
    \sum_{i=1}^N \ell(m_i, x_i) + 
    \lambda \sum_{i=1}^{N-1} I[m_i \neq m_{i+1}].
\end{equation}
The loss function $\ell(m, x)$ is typically the negative log likelihood of the parameter $m$ given the data $x$; smaller loss values indicate a better fit. 
The indicator function $I$ returns 1 if there is a change between positions $i$ and $i+1$, and it returns 0 otherwise. 
The penalty parameter $\lambda\geq 0$ controls the number of detected changepoints 
(small $\lambda$ results in an overfit model with too many changepoints, large $\lambda$ results in an underfit model with too few changepoints).
The penalty $\lambda$ can be selected using theoretically-motivated unsupervised criteria such as AIC/BIC \citep{Akaike73, Schwarz78, Yao88, mBIC}, or using supervised learning algorithms in the labeled data setting \citep{HOCKING-penalties}.

Although the loss $\ell$ is typically convex, the indicator functions $I$ are non-convex, so the optimization problem is non-convex, and gradient-based algorithms can not be used. 
Instead, there are efficient dynamic programming algorithms which can compute a global optimum in quadratic $O(N^2)$ time \citep{segment-neighborhood,Jackson2005}. 
More recently, functional pruning algorithms such as FPOP \citep{Maidstone2016} have been used to compute a global optimum, with time complexity that is also quadratic $O(N^2)$ in the worst case, but log-linear $O(N\log N)$ empirically. 

\begin{table}[t]
    \centering
    \begin{tabular}{ccc}
    \hline
         & Best model with $K$ segments & Best model for penalty $\lambda$  \\
         \hline
    No label constraints & Segment Neighborhood & Optimal Partitioning \\
    & \citep{segment-neighborhood} & \citep{Jackson2005} \\
    \hline
    Label constraints & SegAnnot & LOPART \\
    & \citep{Hocking2014} & \textbf{This paper}\\
    \hline
    \end{tabular}
    \caption{Relationship of the proposed LOPART algorithm to previous dynamic programming algorithms for changepoint detection (rows for constraints, columns for problem formulation). 
    The previous SegAnnot algorithm has the same label constraints but a different problem formulation which makes it impossible to predict new changepoints in unlabeled regions (may have test errors). 
    The previous Optimal Paritioning algorithm has the same problem formulation which can predict new changepoints in unlabeled regions, but ignores the given labels (may have train errors).
    }
    \label{tab:algos}
\end{table}

\paragraph{Novelty with respect to previous work.} Our paper proposes a new changepoint detection algorithm for the case of partially labeled data sequences.
The labeled data setting arises in the context of interactive systems such as SegAnnDB \citep{Hocking2014} and CpLabel \citep{Ford2020} which allow users to view the data and drag with the mouse to define regions with/without significant changepoints in subsets of the data.
The previous OPART and FPOP algorithms can be characterized as ``unsupervised'' because the labels are not used, so they may be inconsistent with the labels (train errors).
In contrast the previous SegAnnot algorithm has constraints which ensure the changepoints are consistent with the labels \citep{Hocking2014}, but there are no other changepoints outside the labels (test errors). 
In this paper we resolve both issues, resulting in our new Labeled Optimal Partitioning (LOPART) algorithm which is more accurate in terms of both train and test errors (Section~\ref{sec:accuracy}). 
The novelty of LOPART is that it combines the label constraints of SegAnnot with the penalized formulation of OPART (Table~\ref{tab:algos}).

\section{Changepoint model for labeled data}

\subsection{Labeled data setting}

In the context of changepoint detection we have a sequence of $N$ data points to segment,
$\mathbf x = [ x_1 \cdots x_N ]$, and the goal is to predict a set of positions $f(\mathbf x)\subseteq \{1, \dots, N-1\}$ with significant changepoints immediately after.
For example, we will treat the simple case of real-valued univariate data $x_1, \dots, x_N \in \mathbb R$ which occur in settings such as detection of copy number changes in cancer genomics \citep{HOCKING-penalties, Hocking2014}. In this setting we typically use the square loss $\ell(m, x)=(m-x)^2$ for a predicted mean parameter $m$ and a data value $x$. However we note that it is straightforward to generalize our algorithm to other kinds of data, by changing the loss function.

We assume the same kind of supervision/labels as were used with the previous SegAnnot algorithm \citep{Hocking2014}. We have a set of $M$ labels
$\mathcal Y=\{(\underline p_j, \overline p_j, y_j)\}_{j=1}^M$. Each label $j\in\{1,\dots, M\}$ has three attributes:
$\underline p_j\in\{1, \dots, N-1\}$ is the start of a labeled region,
$\overline p_j \in \{2, \dots, N\}$ is the end of a labeled region, and $y_j\in\{0,1\}$ is the number of changes expected in the region
$[\underline p_j, \overline p_j]$. 
Note that we could generalize the algorithm to support other kinds of $y_j$ values, but in this paper we only study 0/1 labels.
In this labeled data setting we want a changepoint prediction $f(\mathbf x)\subseteq \{1, \dots, N-1\}$ which minimizes the number of incorrectly predicted labels.

For example
$\mathcal Y=\{
(\underline p_1=1,\overline p_1=2,c_1=0),
(\underline p_2=4,\overline p_2=7,c_2=1)
\}$ 
means that there is no change after the first data point, 
there can be 0--2 changes between data points 2 and 4 (four possibilities: no changes, change after 2, change after 3, or change after both),
and there must be exactly one change somewhere
between data points 4 and 7 (three possibilities). 
A more complex example with $M=3$ labels is shown in Figure~\ref{fig:signal-cost}. 
This example shows how the labels are typically used to encode prior knowledge about the expected/desired changepoints. 
A positive label is typically used in a region with a change of low signal/noise ratio (e.g. a change in mean from 7 to 8 after position 50).
A negative label is typically used in a region with outliers (e.g. at position 86).

For the remainder of the paper we assume the
labeled regions are ordered:
\begin{equation}
  \label{eq:sorted}
  1 \leq 
\underline p_1 < \overline p_1 \leq 
\underline p_2 < \overline p_2 \leq
\cdots \leq 
\underline p_M < \overline p_M \leq 
N.
\end{equation}
If this is not the case, we can sort them in log-linear $O(M\log M)$ time using standard algorithms.
The number of possible labels is $M\in\{0, 1, \dots, N-1\}$.

\begin{figure}
    \input{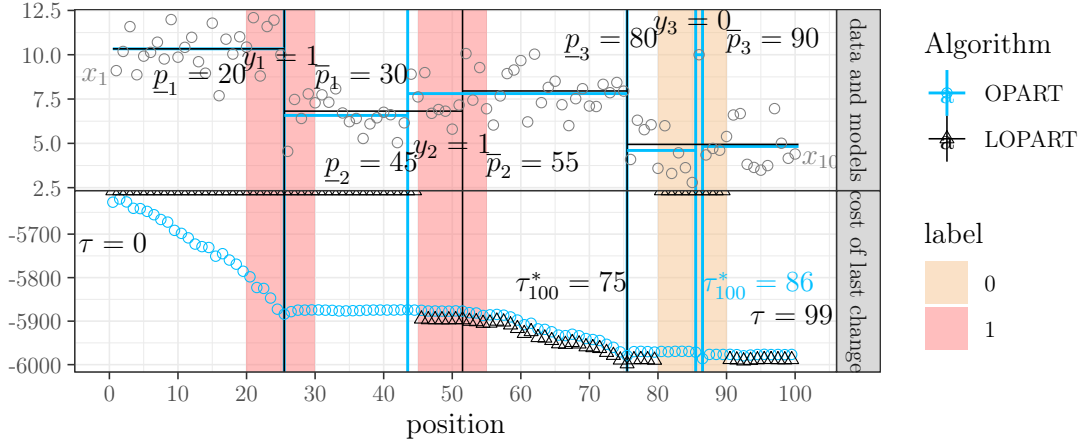}
    \vskip -0.5cm
    \caption{
    Example with $N=100$ data (grey circles) and $M=3$ labels (colored rectagles), 
    showing the novelty of the proposed LOPART algorithm (black) with respect to the classical OPART algorithm (blue). Both algorithms were run with the square loss and a penalty of $\lambda=10$.
    \textbf{Top:} LOPART changepoints (vertical lines) are consistent with 
    all three labels, whereas the OPART model is inconsistent with the second label (should have $y_2=1$ change but OPART predicts 0) and the third label (should have $y_3=0$ changes but OPART predicts 2).
    \textbf{Bottom:} when computing the optimal cost up to position $t=100$ the dynamic programming algorithms minimize over each possible last changepoint $\tau$ using (\ref{eq:op-update}) for OPART and (\ref{eq:lopart-update}) for LOPART (infeasible $\tau$ are shown with infinite cost at the top of the panel). 
    }
    \label{fig:signal-cost}
\end{figure}

\subsection{Optimization problem with label constraints}

The main new idea of our model is to add constraints to original optimal partitioning problem~(\ref{eq:op}) in order to ensure that the changepoints predicted by the model are consistent with the labels. This is similar to the idea of SegAnnot \citep{Hocking2014}, which adds constraints based on the labels to the segment neighborhood problem (Table~\ref{tab:algos}).

To determine whether or not the predicted changepoints are consistent with the given labels, we need to count the number of predicted changes in each labeled region $[\underline p_j, \overline p_j]$. To do that we define 
$$
H(\underline p, \overline p, \mathbf m) =
    \sum_{i=\underline p}^{\overline p-1}
    I[m_i \neq m_{i+1}],
$$
which counts the number of changepoints in the mean vector $\mathbf m$.
If the predicted number of changes $H(\underline p_j, \overline p_j, \mathbf m)$ is equal to the expected number of changes $y_j$, then the model $\mathbf m$ is considered to be consistent with the label $j$.
To define the optimization problem that we would like to solve, we first define an abbreviation for the cost function, which is the same as in the previous problem~(\ref{eq:op}). The cost of a mean vector $\mathbf m$ with penalty $\lambda$ from data point $\underline p$ to data point $\overline p$ is 
$$
    \mathcal C(\underline p, \overline p, \mathbf m, \mathbf x, \lambda) = 
    \sum_{i=\underline p}^{\overline p} 
    \ell(m_i, x_i) 
    +\lambda
    H(\underline p, \overline p, \mathbf m).
$$
Now we can define the labeled optimal partitioning problem using this cost function and a constraint for each label,
\begin{align}
 \min_{
  \mathbf m\in\mathbb R^{N}
  } &\ \ 
  \label{eq:labeled_problem_cost}
\mathcal C(1, N, \mathbf m, \mathbf x, \lambda)
\\
    \text{subject to} 
& \ \ \text{ for all } j\in\{1,\dots,M\},\, 
H(\underline p_j, \overline p_j, \mathbf m)=y_j.
\label{eq:labeled_problem_constraints}
\end{align}
There is one constraint per label $j\in\{1,\dots, M\}$ (\ref{eq:labeled_problem_constraints}),
and each constraint ensures that the labeled number of changes $y_j$ is predicted between $\underline p_j$ and $\overline p_j$.

\section{New Dynamic Programming Algorithm}

Our main contribution is the first algorithm which computes an optimal solution to problem~(\ref{eq:labeled_problem_cost}). In this section we first present some related sub-problems which also need to be solved, then prove the dynamic programming update rules, and finally give pseudocode for the algorithm.

\subsection{Related optimization problems}

Because our algorithm is based on ideas used to solve the optimal partitioning problem~(\ref{eq:op}), we first review that algorithm \citep{Jackson2005}. We need to compute the optimal loss given a single segment with mean parameter $\mu$ starting at $\underline p$ and ending at $\overline p$, which is 
\begin{equation}
    \label{eq:L}
L(\underline p, \overline p, \mathbf x) = 
    \min_{\mu\in\mathbb R}
    \sum_{i=\underline p}^{\overline p}
    \ell(\mu, x_i).
\end{equation}
Note that for many data types and loss functions, one optimal loss value $L(\underline p, \overline p, \mathbf x)$ can be computed in constant $O(1)$ time (e.g. with real-valued data and the square loss, given cumulative sums of the data). 
The dynamic programming algorithm recursively computes the optimal cost in terms of the last changepoint $\tau$,
\begin{equation}
  \hat C_N
  = \min_{\tau\in \{0, 1, \dots, N-1\} }
  \hat C_\tau +
  \lambda +
  L(\tau+1, N, \mathbf x).
  \label{eq:op-update}
\end{equation}
Note that for $\tau=0$, there is only one segment (no changepoints), and we let $\hat C_0=-\lambda$ so that in (\ref{eq:op-update}) we can write the optimal cost in the same way for each value of $\tau$. 
In this paper we propose an algorithm for solving (\ref{eq:labeled_problem_cost}) based on similar ideas. The novelty of our algorithm is that it also accounts for the label constraints (\ref{eq:labeled_problem_constraints}), which reduce the space of possible changepoint $\tau$ values that we need to search. 
To be clear about which constraints are involved in each sub-problem that we need to solve, we first define for any data point $t\in\{1,\dots, N\}$ the index of the last label that we need to consider when computing the cost up to that data point,
\begin{equation}
    \label{eq:J_t}
    J_t = \max \{0\} \cup \{j: \underline p_j < t\}.
\end{equation}
We can then define the cost of the model up to $t$ data points that is consistent with all of the labels up to $J_t$,
\begin{align}
 C_t = \min_{
  \mathbf m\in\mathbb R^{t}
  } &\ \ 
  \label{eq:C_t}
\mathcal C(1, t, \mathbf m, \mathbf x, \lambda)
\\
    \text{subject to} 
& \ \ \text{ for all } j \leq J_t,\, 
H(\underline p_j, \overline p_j, \mathbf m)=y_j.
\end{align}
It is clear that $C_N$ as defined by (\ref{eq:C_t}) is equivalent to the original problem we want to solve (\ref{eq:labeled_problem_cost}). 
However it does not admit a simple recursion as with $\hat C_N$ in (\ref{eq:op-update}).
For example, consider $N=3$ data with $M=1$ label that forces exactly one change $(\underline p_1=1, \overline p_1=3, y_1=1)$. 
The optimal cost can be written as
\begin{eqnarray}
  C_3
  &=& \min \begin{cases}
  L(1, 1, \mathbf x)+L(2, 3, \mathbf x) + \lambda 
  & \text{ if } \tau=1\\  
  L(1, 2, \mathbf x)+L(3, 3, \mathbf x) + \lambda
  & \text{ if } \tau=2.
  \end{cases}
  \label{eq:C_3}
\end{eqnarray}
First note that $\tau=0$ (no changepoint) is not a possibility because of the label constraint. 
Also note that if $\tau=1$ the cost can be written recursively, i.e. $C_3 = C_1 + L(2, 3, \mathbf x) + \lambda$. 
However for $\tau=2$ there is no recursion involving $C_2$, and in fact $C_2$ is not even well-defined, because the constraint is $H(1,3,\mathbf m)=I[m_1\neq m_2]+I[m_2\neq m_3]=1$ but in (\ref{eq:C_t}) there are only two optimization variables $m_1,m_2$ ($m_3$ is undefined).

To resolve this issue we need another optimization problem which ensures that there are no changes in the most recent label. 
Therefore we define another optimal cost value $V_t$ which we will compute for all $t$ that occur inside the labeled regions,
\begin{align}
 V_t = \min_{
  \mathbf m\in\mathbb R^{t}
  } &\ \ 
  \label{eq:V_t_cost}
\mathcal C(1, t, \mathbf m, \mathbf x, \lambda)
\\
    \text{subject to} 
& \ \ \text{ for all } j < J_t,\, 
H(\underline p_j, \overline p_j, \mathbf m)=y_j,
\label{eq:V_t_prev_constraints}\\
\text{and} 
& \ \ 
H(\underline p_{J_t}, t, \mathbf m)=0.
\label{eq:V_t_recent_constraint}
\end{align}
Note that this optimization problem (\ref{eq:V_t_cost}) has the same objective function as (\ref{eq:C_t}) but two kinds of constraints. The first $J_t-1$ constraints (\ref{eq:V_t_prev_constraints}) ensure that the model is consistent with all labels before label $J_t$. The last constraint (\ref{eq:V_t_recent_constraint}) ensures that there are no changes in the current label $J_t$. 
Continuing the example with $N=3$ data points above (\ref{eq:C_3}), we see that $C_3$ can be written in terms of $C_1$ and $V_2$:
\begin{eqnarray}
  C_3
  &=& \min \begin{cases}
  C_1 +L(2, 3, \mathbf x) + \lambda 
  & \text{ if } \tau=1\\  
  V_2 +L(3, 3, \mathbf x) + \lambda
  & \text{ if } \tau=2.
  \end{cases}
  \label{eq:C_3_recursive}
\end{eqnarray}
This example shows that in this case, to compute the final optimal cost $C_N$, we need to compute either $C_t$ or $V_t$ for each $t<N$.
In the next section we prove that this logic can be used for any set of labeled data.

\subsection{Dynamic programming update rules}

To state the rules of the dynamic programming algorithm, we first need to define the optimal cost that we will result in a recursion. For any data point $t\in\{1, \dots, N\}$ we define the optimal cost to be
\begin{equation}
    W_t = \begin{cases}
    V_t & \text{ if } t\in \{\underline p_j + 1, \dots, \overline p_j - 1\} \text{ for some label } j \\
    C_t & \text{ otherwise.}
    \end{cases}
\end{equation}
This definition uses the $V_t$ cost (current label constrained to have no changes) for data points $t$ inside labels, and the standard cost $C_t$ otherwise (model must be consistent with all previous labels). In particular it is clear that $W_N=C_N$ is equivalent to the optimal cost of all $N$ data that we would like to compute (\ref{eq:labeled_problem_cost}).
We also need to define the set of previous changepoints that we will search over to compute the optimal cost. We define this set recursively, starting with $T_0=\{\}$ (empty set) and then for any $t\in\{1, \dots, N\}$:
\begin{equation}
    T_t = \begin{cases}
    T_{t-1} & \text{ if } 
    \exists j: y_j=0 \text{ and } 
    t\in \{\underline p_j + 1, \dots, \overline p_j \}\\
    T_{t-1} & \text{ if } 
    \exists j: y_j=1 \text{ and } 
    t\in \{\underline p_j + 1, \dots, \overline p_j - 1\}\\
    \{\underline p_j, \dots, t-1\} & \text{ if } 
    \exists j: y_j=1 \text{ and } 
    t= \overline p_j \\
    T_{t-1}\cup \{t-1\} & \text{ otherwise.} 
    \end{cases}
    \label{eq:T_t}
\end{equation}
The first two cases of this definition require no changepoints to be added to the set, for all data points $t$ that are inside a label (or at the end of a negative $y_j=0$ label). 
The third case only applies to data points $t$ that occur at the end of a positive $y_j=1$ label, and reinitializes the set to positions within that label. 
The final case is used for data points $t$ in unlabeled regions (or at the start of a label), and adds one possible changepoint at the previous data point $t-1$.
We can now give the following definition for the dynamic progamming update rules.
\begin{definition}[Dynamic programming algorithm for labeled optimal partitioning]
The cost is initialized $\tilde W_0 = -\lambda$ and dynamic programming updates can be computed for any $t\in\{1, \dots, N\}$ via
\begin{equation}
    \tilde W_t = \min_{\tau \in T_t} \tilde W_\tau + \lambda + L(\tau+1, t, \mathbf x).
    \label{eq:lopart-update}
\end{equation}
\end{definition}
Note the similarity with the update rules for the unconstrained problem~(\ref{eq:op-update}). 
In fact the only difference is optimization of the last changepoint $\tau$ over $T_t$ rather than $\{0, 1, \dots, t-1\}$. 
If there are no labels, then $\tilde W_t=\hat C_t$ and $T_t=\{0, 1, \dots, t-1\}$ for all $t$, so the unconstrained (\ref{eq:op-update}) and constrained (\ref{eq:lopart-update}) dynamic programming update rules are identical in this case. 
In general, we have the following theorem which proves the optimality of the recursive dynamic progamming update rules.
\begin{theorem}
The recursively computed cost $\tilde W_t$ is equal to the optimal cost $W_t$ for all data points $t\in\{1, \dots, N\}$.
\end{theorem}
\begin{proof}
The proof is by induction. 
The base case is $t=1$ for which the set of changepoints is $T_1=\{0\}$ and the recursive cost is $\tilde W_1=\tilde W_0 + \lambda + L(1, 1, \mathbf x)  =C_1=W_1$.

Now for any $t\in\{2,\dots, N\}$, we assume that for all $\tau<t$ we have $\tilde W_\tau=W_\tau$ (induction hypothesis), and we aim to prove that $\tilde W_t = W_t$.
We proceed by considering the different cases that are possible. 
\paragraph{Case 1: inside a labeled region.} 
We assume $t\in \{ \underline p_j+1, \dots, \overline p_j-1\}$ for some label $j=J_t$, so $W_t=V_t$ is the optimal cost subject to no changes from $\underline p_j$ to $t$, i.e. $H(\underline p_j, t, \mathbf m)=0$ from (\ref{eq:V_t_recent_constraint}) which implies a upper bound on the last changepoint $\tau<\underline p_j$. 
If there are no previous positive labels then the set of possible last changes is $\{0, \dots, \underline p_j - 1\}\setminus \mathcal A^0$ where $A^0=\cup_{k:y_k=0}\{\underline p_k, \dots, \overline p_k-1\}$ is the set of all negative labeled regions. 
If there is at least one previous positive label $k$ then the set of possible last changes is $\{\underline p_{k}, \dots, \underline p_j - 1\}$. 
In both cases the set of possible changes is equal to the recursively defined set $T_t$. 
For any $\tau\in T_t$ we have $W_\tau=V_\tau$ if $\tau\in\{ \underline p_j+1,\dots, \overline p_j-1\}$ for some label $j$, and $W_\tau=C_\tau$ otherwise.
Therefore the optimal cost can be written as $W_t=V_t=\min_{\tau\in T_t} W_\tau + \lambda + L(\tau+1, t, \mathbf x)$, which by the induction hypothesis equals $\min_{\tau\in T_t} \tilde W_\tau + \lambda + L(\tau+1, t, \mathbf x) = \tilde W_t$.
\paragraph{Case 2: outside a labeled region.}
We assume $t\not \in \{ \underline p_j+1, \dots, \overline p_j-1\}$ for any label $j$, so $W_t=C_t$ is the optimal cost subject to all previous labels.
If there are no previous positive labels then the set of possible last changes is $\{0, \dots, t-1\}\setminus \mathcal A^0$.
If there is at least one previous positive label $k$ then the set of possible last changes is $\{\underline p_{k}, \dots, t - 1\}\setminus \mathcal A^0$.
In both cases this set of possible last changes is equal to the recursively defined set $T_t$. Therefore, using an argument analogous to case 1, the optimal cost is $W_t=C_t=\tilde W_t$, which completes the proof of optimality of the recursive update rules.
\end{proof}

\subsection{Pseudocode, implementation, complexity}


Algorithm~\ref{algo:lopart} (LOPART) inputs a data vector $\mathbf x$, a non-negative penalty parameter $\lambda$, and a set of $M$ labels which are assumed to be sorted in increasing order (line~\ref{line:inputs}). 
On line~\ref{line:init} the algorithm initializes the cost $W_0$ and possible changepoints $T_0$. 
The for loop on line~\ref{line:for} implements the dynamic programming for all data points $t$ from 1 to $N$. 
Since any changepoint $\tau\in\mathcal A^0$ (in a negative label) never appears in any set $T_t$ (\ref{eq:T_t}), we can further optimize the algorithm by running the dynamic programming computations of $T_t,W_t,\tau_t^*$ for $t\not \in A^0$ (outside of negative labels).
Line~\ref{line:T_update} updates the set of possible changepoints using (\ref{eq:T_t}). 
Line~\ref{line:W_update} implements update rule~(\ref{eq:lopart-update}), storing the optimal cost in $W_t$ and the optimal last changepoint in $\tau^*_t$. 
Overall the algorithm is similar to the original optimal partitioning algorithm, but with a more complex update rule on line~\ref{line:T_update} (which exploits the structure of the labels).

\begin{algorithm2e}[H]
\SetAlgoLined
 \caption{Labeled Optimal Partitioning (LOPART)}\label{algo:lopart}
Input: Data $\mathbf x\in\mathbb R^N$, 
penalty $\lambda\in[0, \infty)$, 
labels $\mathbf y\in\{0,1\}^M$, 
positions $\mathbf{\underline p}, \mathbf{\overline p}$ such that
  $1 \leq 
\underline p_1 < \overline p_1 \leq 
\cdots \leq 
\underline p_M < \overline p_M \leq N $
\; \label{line:inputs}
Initialization: $W_0 \gets -\lambda$, $T_0=\{\}$ \; \label{line:init} 
Dynamic progamming: \For{$t=1$ to $N$}{ \label{line:for} 
$T_t\gets \textsc{update}(T_{t-1}, \mathbf y, \mathbf{\underline p}, \mathbf{\overline p})$ // using (\ref{eq:T_t}) \; \label{line:T_update}
$W_t,\tau^*_t \gets \min, \argmin_{\tau\in T_t} W_\tau + \lambda + L(\tau+1, t, \mathbf x)$\; \label{line:W_update}
}
\end{algorithm2e}

\paragraph{Example and comparison with classical OPART.} Consider the example with $N=100$ data and $M=3$ labels shown in Figure~\ref{fig:signal-cost}. 
The classical OPART algorithm ignores the labels, so at $t=100$ it computes the optimal cost by minimizing over all possible last changepoints $\tau\in\{0,\dots, 99\}$, and finds that $\tau^*_{100}=86$ is optimal. 
This changepoint from the outlier $x_{86}=10.0$ is in a negative $y_3=0$ label so the resulting model is inconsistent with this label.
It is also inconsistent with the second positive $y_2=1$ label because the model predicts no changepoints between $\underline p_2=45$ and $\overline p_2=55$.
In constrast the LOPART algorithm computes the optimal cost by minimizing over the constrained set of changepoints $T_{100}=\{45, \dots, 79, 90, \dots, 99\}$ and finds that $\tau^*_{100}=75$ is optimal. 
The resulting model has changepoints that are consistent with all of the labels.

\paragraph{Computational complexity.}
Computing each $T_t$ update on line~\ref{line:T_update} is amortized constant $O(1)$ time on average, but linear $O(N)$ time in the worst case (for $t=\overline p_j=N$ when there is a single positive label $j$ spanning the entire data sequence). 
Computing each minimization on line~\ref{line:W_update} is takes $O(|T_t|)$ time, which is $O(N)$ in the worst case (for $t=N$ when there are no labels). 
The total number of operations over all iterations of the for loop is $\sum_{t=1}^N |T_t|$ which is $O(N)$ in the best case (labels covering the entire data sequence) and $O(N^2)$ in the worst case (no labels).
The space complexity of the algorithm is $O(N)$.

\paragraph{Implementation details.} 
Overall the algorithm can be efficiently implemented in standard C using arrays.
The set $T_t$ can be implemented using an array of size $N$. 
Only the most recent set $T_t$ must be stored (previous sets $T_\tau$ for $\tau<t$ can be discarded), and only the first $|T_t|$ elements of the array are used.
Optimal cost $W$ and last changepoint $\tau^*$ vectors can also be implemented using arrays.
Optimal segment mean parameters, e.g. $\mu$ in (\ref{eq:L}) from solving $L(\tau+1,t,\mathbf x)$, can be computed and stored during the dynamic programming for loop at no extra computational complexity.
As in the original optimal partitioning algorithm, the overall optimal changepoints can be computed by examining the values of the $\tau^*$ vector starting with $\tau^*_N$. 

\paragraph{Implementation for infinite penalty.} LOPART defines a path of optimal models. At one extreme with penalty $\lambda=0$ we have changes in all unlabeled regions. The model at the other extreme has a change in each positive label, and no changes elsewhere. 
This model can be computed when the user inputs infinite penalty $\lambda=\infty$, which can be treated as a special case. First we create a new set of labels, by keeping only the positive labels, and putting negative labels elsewhere. Then we run Algorithm~\ref{algo:lopart} with no penalty $\lambda=0$ to get the optimal changepoints (the optimal cost is infinite).

\paragraph{Previous algorithms which can be used to solve special cases.} In the trivial case of $M=0$ labels, the LOPART optimization problem~(\ref{eq:labeled_problem_cost}) is the same as the classic optimal partitioning problem~(\ref{eq:op}) which can be solved by the OPART algorithm \citep{Jackson2005}.
Also, when we take an infinite penalty, $\lambda=\infty$, then there are no predicted changes outside of positive labels, and the resulting model can be computed by the SegAnnot algorithm \citep{Hocking2014}.

\begin{figure}
  \input{figure-timings-labels}
    \input{figure-timings}
    \vskip -0.5cm
    \caption{Empirical time complexity in simulated data sets (median line and quartile band computed over several data sets of a given size). 
    \textbf{Left:} with $N=10^5$ data, LOPART takes the same amount of time as OPART for a small number of labels, and the same amount of time as FPOP for a large number of labels.
    \textbf{Right:} When there are $O(N)$ positive labels LOPART takes $O(N\log N)$ time (same as FPOP). Larger label density ($M/N$) reduces constant factors.
    }
    \label{fig:timings}
\end{figure}
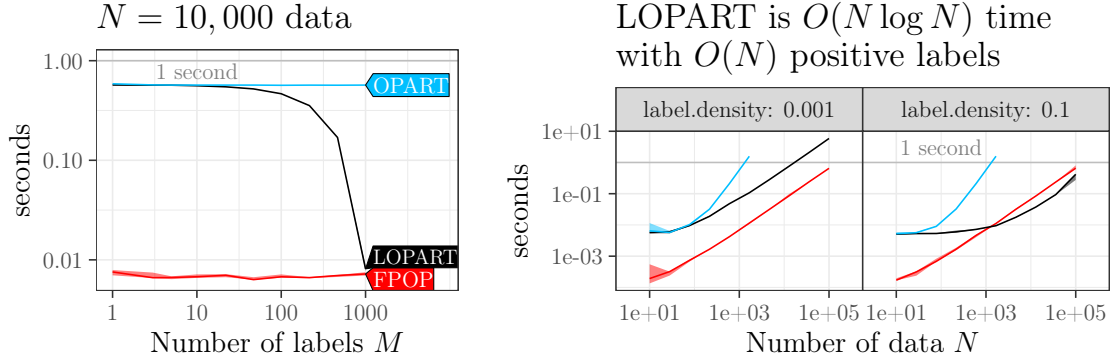


\section{Empirical results}

\subsection{Empirical time complexity in simulated data sets}

As discussed in the previous section, the theoretical/expected time complexity of LOPART is $O(N)$ in the best case (all data labeled) and $O(N^2)$ in the worst case (no data labeled).
To verify this empirically, we conducted timings experiments with simulated data sequences using the standard normal distribution (no changes in mean, but these simulations are only to evaluate time complexity, so they should be representative of real data as well because our algorithm depends only on the number/type of labels, not the data distribution). 
The CPU we used was a 2.40GHz Intel(R) Core(TM)2 Duo CPU P8600.
For baselines we considered the original OPART algorithm which is quadratic $O(N^2)$ time \citep{Jackson2005}, 
and the log-linear $O(N\log N)$ time FPOP algorithm \citep{Maidstone2016}. 
Both baselines compute an optimal solution to the changepoint problem~(\ref{eq:op}) with penalty $\lambda$ and no label constraints. 

In the first experiment, we fixed the data set size at $N=10^5$ (using random normal data as explained in the previous paragraph) and used a variable number of positive labels $M\in\{1, \dots, 1000\}$, each of size 9, every 10 data points. As expected, we observed LOPART timings similar to OPART when the number of labels is small, and timings similar to FPOP when the number of labels is large  (Figure~\ref{fig:timings}, left).
In the second experiment, we fixed the label density $M/N=0.001$ (one positive label per 1000 data points) and varied the number of random normal data $N$. 
In this case LOPART is log-linear $O(N\log N)$ time (same as FPOP) and for $N\geq 1000$ data it showed substantial speedups over the quadratic $O(N^2)$ time OPART (Figure~\ref{fig:timings}, middle).
In the third experiment, we fixed the label density $M/N=0.1$ and varied the number of random normal data $N$. 
As expected with many labels, LOPART is much faster, and in fact faster than FPOP (by constant factors) for $N\geq 1000$ data (Figure~\ref{fig:timings}, right).
Overall these experiments show that LOPART is at least as fast as OPART, and can be substantially faster when there are many labels.

\subsection{Empirical accuracy with respect to labels in real genomic data}
\label{sec:accuracy}
\paragraph{Data sets.} To examine the changepoint prediction accuracy of LOPART, we performed the following experiments in real genomic data. 
For baseline algorithms we considered OPART \citep{Jackson2005} and SegAnnot \citep{Hocking2014}.
Genomic scientists created labels for 413 data sequences from cancer DNA copy number profiles using the SegAnnDB system \citep{Hocking2014}.  
In these data there are separate sequences for each patient and chromosome; abrupt changes in a sequence are important diagnostic markers for aggressive cancer subtypes \citep{gudrun-jclinicaloncology}.
The number of data points per sequence ranges from $N=39$ to 43628, and the number of labels ranges from $M=2$ to 12 (with at least one positive and one negative label per sequence).

\paragraph{Evaluation metrics.} The main evaluation metric that we use is the total number of label errors, which is the sum of false positives and false negatives over all labels $j$ in the train/test sets.
A false positive is a label $j$ such that $H(\underline p_j, \overline p_j, \mathbf m) > y_j$ (more predicted changes than expected for either a positive or negative label),
a false negative is $H(\underline p_j, \overline p_j, \mathbf m) = 0 < y_j=1$ (no predicted changes for a positive label), and
a true positive is $H(\underline p_j, \overline p_j, \mathbf m) \geq y_j=1$ (one or more predicted changes for a positive label).
We also perform Receiver Operating Characteristic (ROC) analysis, which examines the True Positive Rate as a function of the False Positive Rate (different points on the ROC curve are computed using different penalty $\lambda$ values).


\begin{figure}
    \centering
    \includegraphics[width=0.37\textwidth]{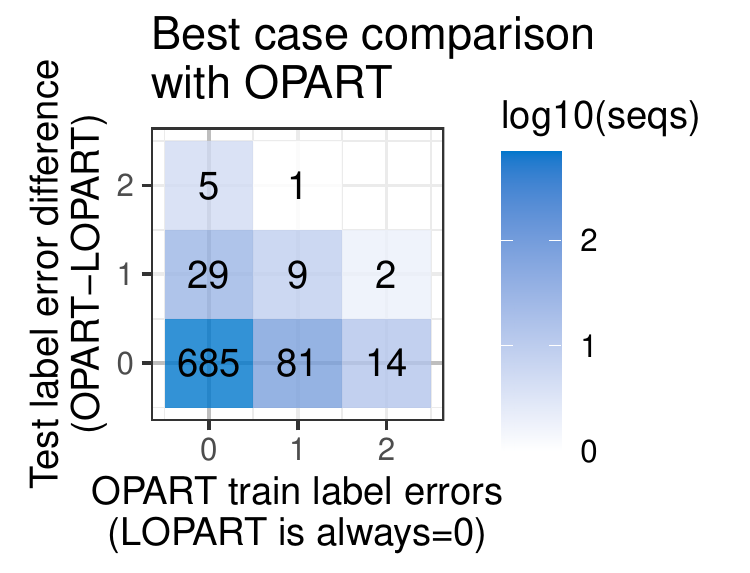}
    \includegraphics[width=0.6\textwidth]{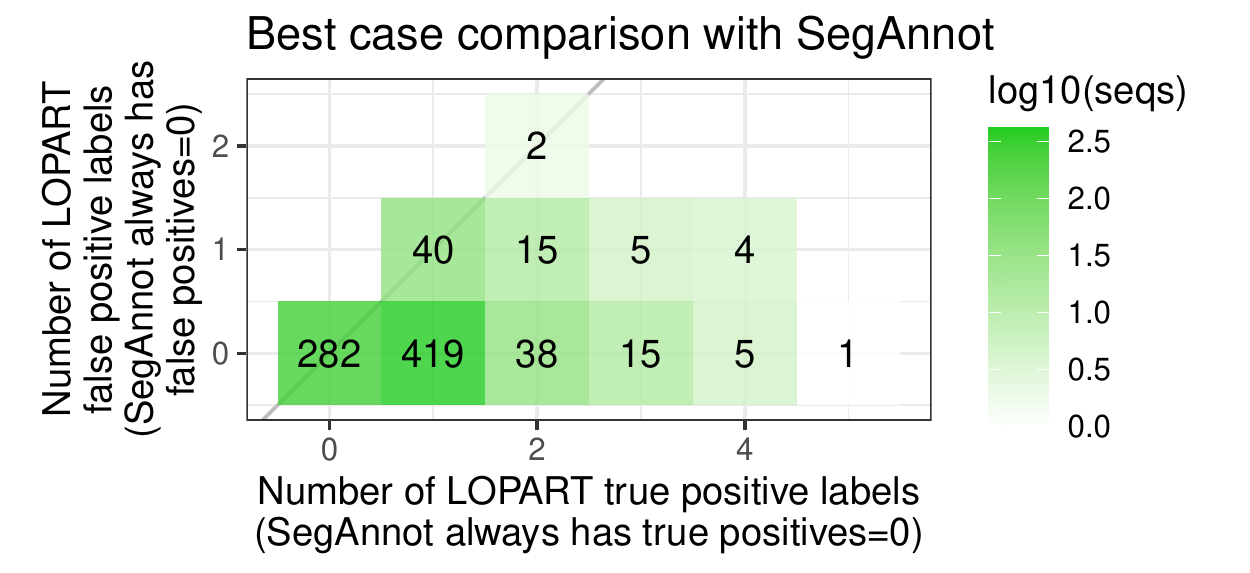}
    \caption{Comparing LOPART with baselines in terms of best case label errors in 2-fold cross-validation on real genomic data (penalty $\lambda$ selected for each algorithm/sequence/split by minimizing the total label errors, train+test). 
    \textbf{Left:} LOPART never has more label errors than OPART (grey horizontal line indicates equal test label errors, vertical line indicates equal train label errors).
    \textbf{Right:} LOPART never has more test errors than SegAnnot (grey diagonal line indicates equal test errors for LOPART and SegAnnot). 
    }
    \label{fig:label-errors}
\end{figure}

\paragraph{Cross-validation setup.} For each data sequence we first randomly assigned each label to a fold ID, and used $K=2$ fold cross-validation to obtain two train/test splits per sequence (each train/test set has at least one label per sequence).
We also tried sequential rather than random assignment (first half of labels on each data sequence are fold 1, second half are fold 2), and we observed qualitatively similar results (same ranking of algorithms), so we report only the results for random assignment below.

\paragraph{Grid of penalty values.} For each data sequence and train/test split we ran LOPART (using only the labels in the train set) and OPART, both with a grid of 21 penalty values evenly spaced on the log scale,  $\lambda\in\{10^{-5}, 10^{-4.5}, \dots, 10^5\}$.

\paragraph{Best penalty analysis.} 
The goal of this analysis is to determine label error differences between algorithms in the best case for each algorithm (i.e. when the penalty is properly chosen).
For each split/sequence/algorithm we chose a penalty which minimized the total number of label errors (train+test), and then we analyzed the train/test error differences between algorithms (Figure~\ref{fig:label-errors}).

\paragraph{Best penalty comparison with OPART/FPOP.} 
Since LOPART has zero train label errors by definition, we expected OPART to have more errors in some cases, even after optimizing over penalty values.
We observed that the best OPART model had 0 train label errors in $719/826=87\%$ of sequences/splits (counts on vertical grey line in Figure~\ref{fig:label-errors}, left), but 1--2 train label errors in $107/826=13\%$ of sequences/splits (counts right of vertical grey line in Figure~\ref{fig:label-errors}, left).
We also compared the number of test label errors per algorithm, after optimizing over penalty values. 
We observed that LOPART had the same number of test label errors in $780/826=94\%$ of sequences/splits (counts on horizontal grey line in Figure~\ref{fig:label-errors}, left), and 1--2 fewer test label errors in $46/826=6\%$ of sequences/splits (counts above horizontal grey line in Figure~\ref{fig:label-errors}, left).
We did not observe any data sets or splits for which LOPART had more train or test label errors than OPART (after optimizing over penalty values).
These data indicate that after optimizing over penalty values LOPART is always at least as accurate as OPART in these real data, and LOPART is sometimes more accurate. These conclusions also hold for FPOP, because it computes the same optimal solution as OPART.



\paragraph{Best penalty comparison with SegAnnot.} LOPART and SegAnnot both have constraints that ensure zero label errors with respect to the train set, so we compared them by computing the number of label errors with respect to the test set (after optimizing over penalty $\lambda$ values for LOPART; SegAnnot is equivalent to always taking penalty $\lambda=\infty$ in LOPART).
SegAnnot never predicts any changes in unlabeled regions, so it always has zero false positives and maximal false negatives with respect to the test labels.
We expected LOPART to have decreased label error rates due to decreased false negative rates (it can predict changepoints in unlabeled regions).
In $324/826=39\%$ of sequences/splits LOPART and SegAnnot had the same number of test errors (counts on diagonal grey line in Figure~\ref{fig:label-errors}, right).
In $502/826=61\%$ of sequences/splits LOPART had fewer test errors than SegAnnot (more true positives than false positives, counts below diagonal grey line in Figure~\ref{fig:label-errors}, left).
We did not observe any sequences/splits for which LOPART had more test label errors than SegAnnot.
Overall these data indicate that LOPART with best penalty is always as accurate as SegAnnot, and frequently more accurate in these real genomic data sets.

\paragraph{Predicted penalty analysis.} 
The main goal of this analysis is to determine the extent to which a penalty learned using OPART can be used for prediction with LOPART.
The LOPART algorithm has no train label errors for any penalty $\lambda$, because it uses the train labels in the definition of its optimization problem~(\ref{eq:labeled_problem_cost}).
To choose the penalty $\lambda$ to use with LOPART, we propose to learn a penalty using OPART (which does not use the train labels in its optimization problem, so it may have train label errors). 
To do this we first run OPART for several penalty values, then we compute label error rates for each penalty/sequence/split.
We then use three different methods for learning/predicting the penalty $\lambda$ to use for each test data sequence:
\begin{description}
\item[BIC.0] uses the classical Bayesian Information Criterion of \citet{Schwarz78}, which means predicting $\lambda_i=\log N_i$ for each data sequence $i$, where $N_i$ is the number of data points to segment (this is unsupervised since it ignores the labels; 0 learned parameters).
\item[constant.1] uses grid search to choose a penalty value with minimal train label errors, then predicts this constant $\lambda$ for each test data sequence (this is supervised since it uses the labels; 1 learned parameter).
\item[linear.2] uses the linear penalty function learning algorithm of \citet{HOCKING-penalties}, with a single feature $x_i = \log \log N_i$ for each data sequence $i$. To make a prediction $\log \lambda_i  = f(x_i) = w^T x_i + b$ we first learn the weight $w$ and bias $b$ using convex optimization of a squared hinge loss which approximates the train label error (supervised; 2 learned  parameters). The feature $x_i = \log \log N_i$ was chosen to facilitate comparison with the unsupervised BIC penalty, which corresponds to always using $w=1,b=0$ in this model.
\end{description}
We used each predicted penalty value with both OPART and LOPART, then analyzed the test accuracy (Figure~\ref{fig:cv-BIC}) and area under the ROC curve (Figure~\ref{fig:cv-BIC-roc}).

\begin{figure}
    \centering
    \includegraphics[width=\textwidth]{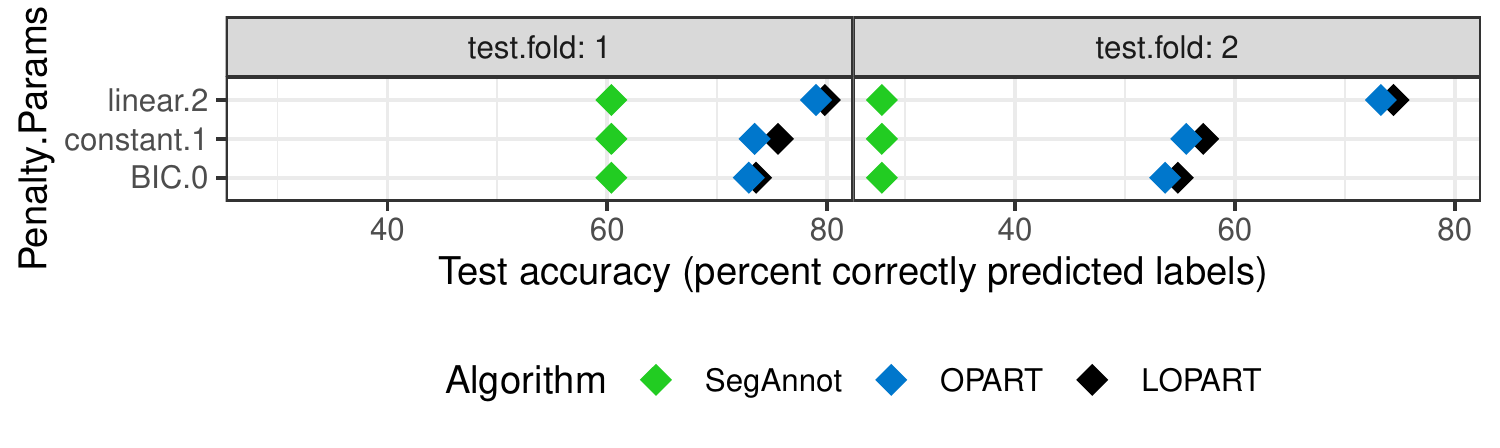}
    \vskip -0.5cm
    \caption{Comparing LOPART with baselines in terms of test accuracy using predicted penalties in real genomic data. Using both unsupervised (BIC.0) and supervised (constant.1, linear.2) penalty prediction methods, LOPART is slightly more accurate than OPART, and much more accurate than SegAnnot (in both test sets).
    }
    \label{fig:cv-BIC}
\end{figure}

\begin{figure}
    \centering
    \includegraphics[width=0.7\textwidth]{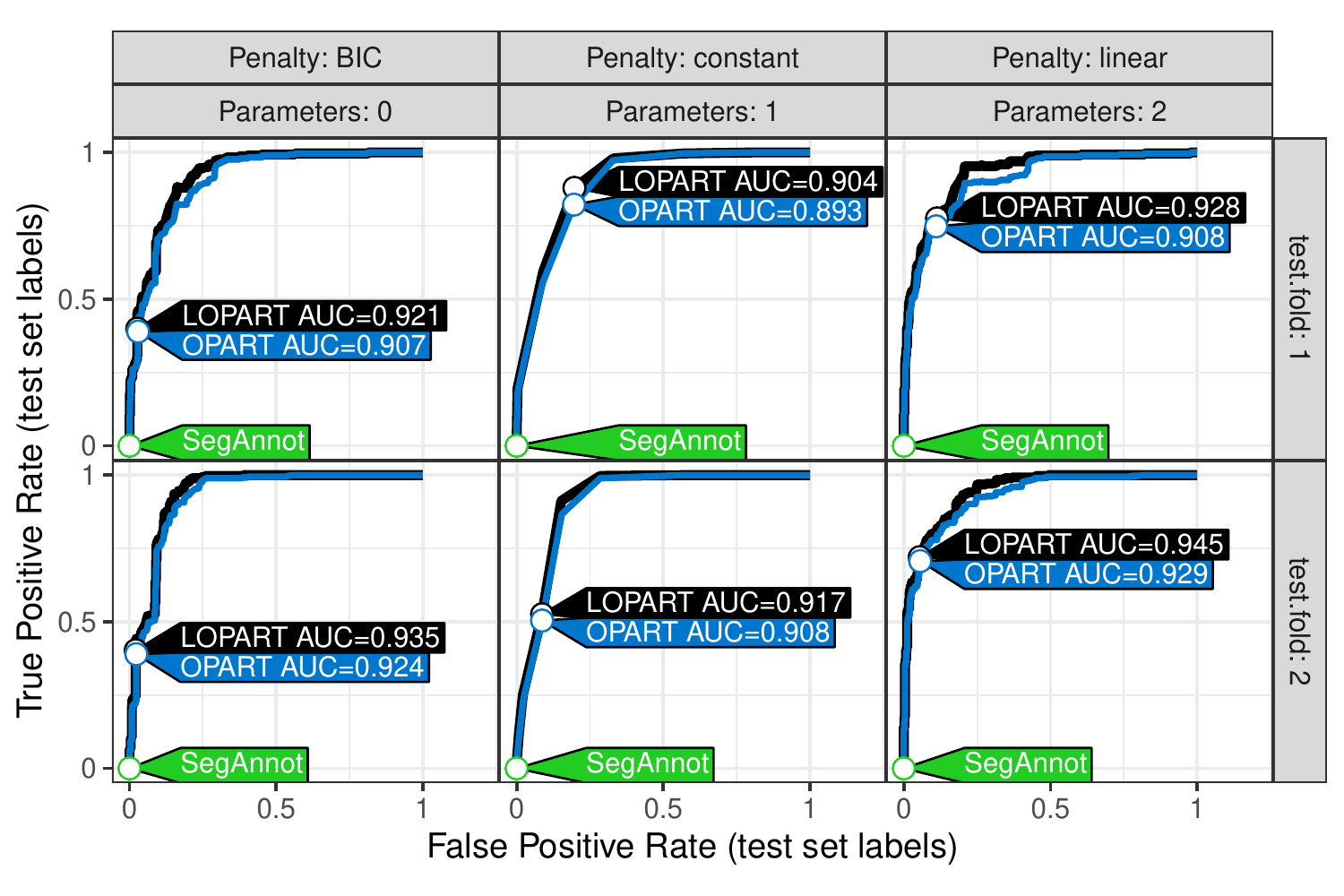}
    \vskip -0.5cm
    \caption{Receiver Operating Characteristic (ROC) analysis of penalty predictions using three methods (panels from left to right) in cross-validation using two folds (panels from top to bottom) in real genomic data. LOPART has consistently larger Area Under the Curve (AUC) than OPART. No ROC curve drawn for SegAnnot because it has no penalty/regularization parameter (never predicts any changepoints in unlabeled test data).
    }
    \label{fig:cv-BIC-roc}
\end{figure}

\paragraph{Predicted penalty comparison with OPART/FPOP.} 
We expected that penalties learned using OPART should result in reasonable predictions using LOPART, because the two algorithms use the penalty in the same way (penalty $\lambda$ added to cost for each changepoint).
Surprisingly, we observed that LOPART is slightly but consistently more accurate than OPART (Figure~\ref{fig:cv-BIC}), with differences of 0.6--2.1\% across the three penalty prediction methods and two test folds.
The ROC analysis also indicates that LOPART is slightly more accurate than OPART (Figure~\ref{fig:cv-BIC-roc}). Over the two test folds and three penalty prediction methods we observed that LOPART had 0.009--0.02 larger AUC than OPART.
These data indicate that OPART/FPOP can be used to learn a penalty for predition with LOPART, and that LOPART has slightly more accurate predictions than OPART/FPOP with the learned penalty.

\paragraph{Predicted penalty comparison with SegAnnot.} 
We expected LOPART with learned penalties to be more accurate than SegAnnot, for the same reasons as in the best penalty comparison (SegAnnot never predicts any changes in unlabeled regions so always has 100\% false negative rate).
In agreement with this expectation, we observed that LOPART has consistently much larger test accuracy rates than SegAnnot (Figure~\ref{fig:cv-BIC}).
Over the two test folds and three penalty prediction methods, we observed improvements of 13--47\% accuracy.
In the ROC analysis, SegAnnot is a single point at TPR=FPR=0\% (Figure~\ref{fig:cv-BIC-roc}).
Overall this analysis indicates that LOPART yields consistently more accurate predictions than SegAnnot in real genomic data. 

\section{Discussion and Conclusions}

We proposed a new algorithm, LOPART, for changepoint detection in a partially labeled sequence of $N$ data.
It combines ideas from Optimal Partitioning \citep{Jackson2005} with SegAnnot \citep{Hocking2014}, which is the only previous changepoint detection algorithm that guarantees consistency with the given labels, but does not predict any changepoints in unlabeled/test regions.
The novelty of LOPART with respect to SegAnnot is the penalized formulation, in which SegAnnot can be viewed as the special case with infinite penalty; decreasing the penalty results in increasing the number of predicted changepoints in unlabeled/test regions.

Our theoretical result proves that that LOPART dynamic programming update rule computes an optimal solution subject to the label constraints in $O(N)$ time in the best case, and $O(N^2)$ in the worst case.
Our empirical timings in simulated data showed that LOPART runs faster with more labels, and actually runs in log-linear $O(N\log N)$ time when the number of positive labels is $O(N)$.
Our empirical accuracy analysis using best penalties in real genomic data showed that LOPART is always at least as accurate as the OPART/FPOP and SegAnnot baselines, and LOPART is often more accurate.
Finally our predicted penalty analysis demonstrated the feasibility of learning a penalty using OPART/FPOP and then using it for prediction using LOPART.
Surprisingly, we observed that LOPART is slightly more accurate than OPART/FPOP, and much more accurate than SegAnnot (using either unsupervised penalties or supervised penalties learned with OPART/FPOP).
These advantages suggest that when a user requires a model that is consistent with the given labels, LOPART should be used rather than SegAnnot.
FPOP may be preferred for its empirical log-linear $O(N\log N)$ complexity when there are many data $N$ and few labels $M$ (although it may have some train label errors).

For future work, we would like to solve the same problem with label constraints~(\ref{eq:labeled_problem_cost}) using inequality pruning \citep{pelt} or functional pruning \citep{Maidstone2016}, which we expect would be faster (log-linear rather than quadratic, even with few labels). 
Furthermore, we could use functional pruning to solve more complex problems with different kinds of labels \citep{HOCKING2016-chipseq} and additional constraints on the directions of changes \citep{Hocking2017}.

\paragraph{Reproducible research statement.}
Our C code that implements LOPART is in a free/open-source R package on GitHub (\url{https://github.com/tdhock/LOPART}).
We also have created a dedicated GitHub repository with the code and data necessary to reproduce our figures and empirical results (\url{https://github.com/tdhock/LOPART-paper}).

\bibliographystyle{abbrvnat}
\bibliography{refs}

\end{document}

%% file: figure-timings-labels.tex
\begin{tikzpicture}[x=1pt,y=1pt]
\definecolor{fillColor}{RGB}{255,255,255}
\path[use as bounding box,fill=fillColor,fill opacity=0.00] (0,0) rectangle (180.67,144.54);
\begin{scope}
\path[clip] (  0.00,  0.00) rectangle (180.67,144.54);
\definecolor{drawColor}{RGB}{255,255,255}
\definecolor{fillColor}{RGB}{255,255,255}

\path[draw=drawColor,line width= 0.6pt,line join=round,line cap=round,fill=fillColor] (  0.00,  0.00) rectangle (180.68,144.54);
\end{scope}
\begin{scope}
\path[clip] ( 38.23, 30.69) rectangle (175.17,121.88);
\definecolor{fillColor}{RGB}{255,255,255}

\path[fill=fillColor] ( 38.23, 30.69) rectangle (175.18,121.88);
\definecolor{drawColor}{gray}{0.92}

\path[draw=drawColor,line width= 0.3pt,line join=round] ( 38.23, 61.19) --
	(175.17, 61.19);

\path[draw=drawColor,line width= 0.3pt,line join=round] ( 38.23, 98.89) --
	(175.17, 98.89);

\path[draw=drawColor,line width= 0.3pt,line join=round] ( 60.40, 30.69) --
	( 60.40,121.88);

\path[draw=drawColor,line width= 0.3pt,line join=round] ( 92.30, 30.69) --
	( 92.30,121.88);

\path[draw=drawColor,line width= 0.3pt,line join=round] (124.20, 30.69) --
	(124.20,121.88);

\path[draw=drawColor,line width= 0.3pt,line join=round] (156.09, 30.69) --
	(156.09,121.88);

\path[draw=drawColor,line width= 0.3pt,line join=round] (172.04, 30.69) --
	(172.04,121.88);

\path[draw=drawColor,line width= 0.6pt,line join=round] ( 38.23, 42.35) --
	(175.17, 42.35);

\path[draw=drawColor,line width= 0.6pt,line join=round] ( 38.23, 80.04) --
	(175.17, 80.04);

\path[draw=drawColor,line width= 0.6pt,line join=round] ( 38.23,117.74) --
	(175.17,117.74);

\path[draw=drawColor,line width= 0.6pt,line join=round] ( 44.45, 30.69) --
	( 44.45,121.88);

\path[draw=drawColor,line width= 0.6pt,line join=round] ( 76.35, 30.69) --
	( 76.35,121.88);

\path[draw=drawColor,line width= 0.6pt,line join=round] (108.25, 30.69) --
	(108.25,121.88);

\path[draw=drawColor,line width= 0.6pt,line join=round] (140.14, 30.69) --
	(140.14,121.88);
\definecolor{drawColor}{RGB}{190,190,190}

\path[draw=drawColor,line width= 0.6pt,line join=round] ( 38.23,117.74) -- (175.17,117.74);
\definecolor{drawColor}{gray}{0.50}

\node[text=drawColor,anchor=base,inner sep=0pt, outer sep=0pt, scale=  0.71] at ( 76.35,110.68) {1 second};
\definecolor{fillColor}{RGB}{255,0,0}

\path[fill=fillColor,fill opacity=0.50] ( 44.45, 38.50) --
	( 59.67, 37.53) --
	( 66.75, 36.12) --
	( 76.35, 36.84) --
	( 87.27, 36.94) --
	( 97.79, 35.78) --
	(108.25, 36.84) --
	(118.92, 35.79) --
	(129.54, 36.62) --
	(140.14, 37.67) --
	(140.14, 36.42) --
	(129.54, 35.88) --
	(118.92, 35.39) --
	(108.25, 35.58) --
	( 97.79, 34.83) --
	( 87.27, 35.78) --
	( 76.35, 35.52) --
	( 66.75, 35.04) --
	( 59.67, 35.46) --
	( 44.45, 36.49) --
	cycle;

\path[] ( 44.45, 38.50) --
	( 59.67, 37.53) --
	( 66.75, 36.12) --
	( 76.35, 36.84) --
	( 87.27, 36.94) --
	( 97.79, 35.78) --
	(108.25, 36.84) --
	(118.92, 35.79) --
	(129.54, 36.62) --
	(140.14, 37.67);

\path[] (140.14, 36.42) --
	(129.54, 35.88) --
	(118.92, 35.39) --
	(108.25, 35.58) --
	( 97.79, 34.83) --
	( 87.27, 35.78) --
	( 76.35, 35.52) --
	( 66.75, 35.04) --
	( 59.67, 35.46) --
	( 44.45, 36.49);
\definecolor{fillColor}{RGB}{0,0,0}

\path[fill=fillColor,fill opacity=0.50] ( 44.45,108.68) --
	( 59.67,108.51) --
	( 66.75,108.55) --
	( 76.35,108.25) --
	( 87.27,107.86) --
	( 97.79,107.07) --
	(108.25,105.27) --
	(118.92,100.83) --
	(129.54, 88.78) --
	(140.14, 40.05) --
	(140.14, 38.02) --
	(129.54, 88.68) --
	(118.92,100.66) --
	(108.25,105.21) --
	( 97.79,106.97) --
	( 87.27,107.80) --
	( 76.35,108.20) --
	( 66.75,108.41) --
	( 59.67,108.39) --
	( 44.45,108.51) --
	cycle;

\path[] ( 44.45,108.68) --
	( 59.67,108.51) --
	( 66.75,108.55) --
	( 76.35,108.25) --
	( 87.27,107.86) --
	( 97.79,107.07) --
	(108.25,105.27) --
	(118.92,100.83) --
	(129.54, 88.78) --
	(140.14, 40.05);

\path[] (140.14, 38.02) --
	(129.54, 88.68) --
	(118.92,100.66) --
	(108.25,105.21) --
	( 97.79,106.97) --
	( 87.27,107.80) --
	( 76.35,108.20) --
	( 66.75,108.41) --
	( 59.67,108.39) --
	( 44.45,108.51);
\definecolor{fillColor}{RGB}{0,191,255}

\path[fill=fillColor,fill opacity=0.50] ( 44.45,109.02) --
	( 59.67,108.67) --
	( 66.75,108.74) --
	( 76.35,108.55) --
	( 87.27,108.53) --
	( 97.79,108.55) --
	(108.25,108.51) --
	(118.92,108.49) --
	(129.54,108.48) --
	(140.14,108.51) --
	(140.14,108.47) --
	(129.54,108.43) --
	(118.92,108.47) --
	(108.25,108.43) --
	( 97.79,108.44) --
	( 87.27,108.45) --
	( 76.35,108.50) --
	( 66.75,108.48) --
	( 59.67,108.47) --
	( 44.45,108.83) --
	cycle;

\path[] ( 44.45,109.02) --
	( 59.67,108.67) --
	( 66.75,108.74) --
	( 76.35,108.55) --
	( 87.27,108.53) --
	( 97.79,108.55) --
	(108.25,108.51) --
	(118.92,108.49) --
	(129.54,108.48) --
	(140.14,108.51);

\path[] (140.14,108.47) --
	(129.54,108.43) --
	(118.92,108.47) --
	(108.25,108.43) --
	( 97.79,108.44) --
	( 87.27,108.45) --
	( 76.35,108.50) --
	( 66.75,108.48) --
	( 59.67,108.47) --
	( 44.45,108.83);
\definecolor{drawColor}{RGB}{255,0,0}

\path[draw=drawColor,line width= 0.6pt,line join=round] ( 44.45, 37.77) --
	( 59.67, 35.59) --
	( 66.75, 35.67) --
	( 76.35, 35.95) --
	( 87.27, 36.51) --
	( 97.79, 34.86) --
	(108.25, 35.94) --
	(118.92, 35.63) --
	(129.54, 36.42) --
	(140.14, 37.01);
\definecolor{drawColor}{RGB}{0,0,0}

\path[draw=drawColor,line width= 0.6pt,line join=round] ( 44.45,108.54) --
	( 59.67,108.47) --
	( 66.75,108.46) --
	( 76.35,108.25) --
	( 87.27,107.85) --
	( 97.79,107.06) --
	(108.25,105.23) --
	(118.92,100.74) --
	(129.54, 88.68) --
	(140.14, 38.93);
\definecolor{drawColor}{RGB}{0,191,255}

\path[draw=drawColor,line width= 0.6pt,line join=round] ( 44.45,108.99) --
	( 59.67,108.57) --
	( 66.75,108.51) --
	( 76.35,108.51) --
	( 87.27,108.50) --
	( 97.79,108.51) --
	(108.25,108.43) --
	(118.92,108.49) --
	(129.54,108.46) --
	(140.14,108.51);
\end{scope}
\begin{scope}
\path[clip] ( 38.23, 30.69) rectangle (175.17,121.88);
\definecolor{drawColor}{RGB}{0,0,0}
\definecolor{fillColor}{RGB}{255,0,0}

\path[draw=drawColor,line width= 0.4pt,line join=round,line cap=round,fill=fillColor] (140.14, 37.01) --
	(142.99, 39.36) --
	(165.94, 39.36) --
	(165.94, 30.69) --
	(142.99, 30.69) --
	cycle;
\definecolor{fillColor}{RGB}{0,0,0}

\path[draw=drawColor,line width= 0.4pt,line join=round,line cap=round,fill=fillColor] (140.14, 38.93) --
	(142.99, 47.60) --
	(175.18, 47.60) --
	(175.18, 39.36) --
	(142.99, 39.36) --
	cycle;
\definecolor{fillColor}{RGB}{0,191,255}

\path[draw=drawColor,line width= 0.4pt,line join=round,line cap=round,fill=fillColor] (140.14,108.51) --
	(142.99,112.84) --
	(171.76,112.84) --
	(171.76,104.17) --
	(142.99,104.17) --
	cycle;
\definecolor{drawColor}{RGB}{255,255,255}

\node[text=drawColor,anchor=base west,inner sep=0pt, outer sep=0pt, scale=  0.70] at (142.99, 32.13) {FPOP};

\node[text=drawColor,anchor=base west,inner sep=0pt, outer sep=0pt, scale=  0.66] at (142.99, 40.74) {LOPART};

\node[text=drawColor,anchor=base west,inner sep=0pt, outer sep=0pt, scale=  0.70] at (142.99,105.61) {OPART};
\definecolor{drawColor}{gray}{0.20}

\path[draw=drawColor,line width= 0.6pt,line join=round,line cap=round] ( 38.23, 30.69) rectangle (175.18,121.88);
\end{scope}
\begin{scope}
\path[clip] (  0.00,  0.00) rectangle (180.67,144.54);
\definecolor{drawColor}{gray}{0.30}

\node[text=drawColor,anchor=base east,inner sep=0pt, outer sep=0pt, scale=  0.73] at ( 33.28, 39.32) {0.01};

\node[text=drawColor,anchor=base east,inner sep=0pt, outer sep=0pt, scale=  0.73] at ( 33.28, 77.01) {0.10};

\node[text=drawColor,anchor=base east,inner sep=0pt, outer sep=0pt, scale=  0.73] at ( 33.28,114.71) {1.00};
\end{scope}
\begin{scope}
\path[clip] (  0.00,  0.00) rectangle (180.67,144.54);
\definecolor{drawColor}{gray}{0.20}

\path[draw=drawColor,line width= 0.6pt,line join=round] ( 35.48, 42.35) --
	( 38.23, 42.35);

\path[draw=drawColor,line width= 0.6pt,line join=round] ( 35.48, 80.04) --
	( 38.23, 80.04);

\path[draw=drawColor,line width= 0.6pt,line join=round] ( 35.48,117.74) --
	( 38.23,117.74);
\end{scope}
\begin{scope}
\path[clip] (  0.00,  0.00) rectangle (180.67,144.54);
\definecolor{drawColor}{gray}{0.20}

\path[draw=drawColor,line width= 0.6pt,line join=round] ( 44.45, 27.94) --
	( 44.45, 30.69);

\path[draw=drawColor,line width= 0.6pt,line join=round] ( 76.35, 27.94) --
	( 76.35, 30.69);

\path[draw=drawColor,line width= 0.6pt,line join=round] (108.25, 27.94) --
	(108.25, 30.69);

\path[draw=drawColor,line width= 0.6pt,line join=round] (140.14, 27.94) --
	(140.14, 30.69);
\end{scope}
\begin{scope}
\path[clip] (  0.00,  0.00) rectangle (180.67,144.54);
\definecolor{drawColor}{gray}{0.30}

\node[text=drawColor,anchor=base,inner sep=0pt, outer sep=0pt, scale=  0.73] at ( 44.45, 19.68) {1};

\node[text=drawColor,anchor=base,inner sep=0pt, outer sep=0pt, scale=  0.73] at ( 76.35, 19.68) {10};

\node[text=drawColor,anchor=base,inner sep=0pt, outer sep=0pt, scale=  0.73] at (108.25, 19.68) {100};

\node[text=drawColor,anchor=base,inner sep=0pt, outer sep=0pt, scale=  0.73] at (140.14, 19.68) {1000};
\end{scope}
\begin{scope}
\path[clip] (  0.00,  0.00) rectangle (180.67,144.54);
\definecolor{drawColor}{RGB}{0,0,0}

\node[text=drawColor,anchor=base,inner sep=0pt, outer sep=0pt, scale=  0.92] at (106.70,  7.64) {Number of labels $M$};
\end{scope}
\begin{scope}
\path[clip] (  0.00,  0.00) rectangle (180.67,144.54);
\definecolor{drawColor}{RGB}{0,0,0}

\node[text=drawColor,rotate= 90.00,anchor=base,inner sep=0pt, outer sep=0pt, scale=  0.92] at ( 13.08, 76.28) {seconds};
\end{scope}
\begin{scope}
\path[clip] (  0.00,  0.00) rectangle (180.67,144.54);
\definecolor{drawColor}{RGB}{0,0,0}

\node[text=drawColor,anchor=base west,inner sep=0pt, outer sep=0pt, scale=  1.10] at ( 38.23,129.95) {$N=10,000$ data};
\end{scope}
\end{tikzpicture}

%% file: figure-timings.tex
\begin{tikzpicture}[x=1pt,y=1pt]
\definecolor{fillColor}{RGB}{255,255,255}
\path[use as bounding box,fill=fillColor,fill opacity=0.00] (0,0) rectangle (238.49,144.54);
\begin{scope}
\path[clip] (  0.00,  0.00) rectangle (238.49,144.54);
\definecolor{drawColor}{RGB}{255,255,255}
\definecolor{fillColor}{RGB}{255,255,255}

\path[draw=drawColor,line width= 0.6pt,line join=round,line cap=round,fill=fillColor] (  0.00,  0.00) rectangle (238.49,144.54);
\end{scope}
\begin{scope}
\path[clip] ( 46.36, 30.69) rectangle (139.68, 91.06);
\definecolor{fillColor}{RGB}{255,255,255}

\path[fill=fillColor] ( 46.36, 30.69) rectangle (139.68, 91.06);
\definecolor{drawColor}{gray}{0.92}

\path[draw=drawColor,line width= 0.3pt,line join=round] ( 46.36, 31.92) --
	(139.68, 31.92);

\path[draw=drawColor,line width= 0.3pt,line join=round] ( 46.36, 55.57) --
	(139.68, 55.57);

\path[draw=drawColor,line width= 0.3pt,line join=round] ( 46.36, 79.21) --
	(139.68, 79.21);

\path[draw=drawColor,line width= 0.3pt,line join=round] ( 76.06, 30.69) --
	( 76.06, 91.06);

\path[draw=drawColor,line width= 0.3pt,line join=round] (109.99, 30.69) --
	(109.99, 91.06);

\path[draw=drawColor,line width= 0.6pt,line join=round] ( 46.36, 43.74) --
	(139.68, 43.74);

\path[draw=drawColor,line width= 0.6pt,line join=round] ( 46.36, 67.39) --
	(139.68, 67.39);

\path[draw=drawColor,line width= 0.6pt,line join=round] ( 46.36, 91.03) --
	(139.68, 91.03);

\path[draw=drawColor,line width= 0.6pt,line join=round] ( 59.09, 30.69) --
	( 59.09, 91.06);

\path[draw=drawColor,line width= 0.6pt,line join=round] ( 93.02, 30.69) --
	( 93.02, 91.06);

\path[draw=drawColor,line width= 0.6pt,line join=round] (126.95, 30.69) --
	(126.95, 91.06);
\definecolor{drawColor}{RGB}{190,190,190}

\path[draw=drawColor,line width= 0.6pt,line join=round] ( 46.36, 79.21) -- (139.68, 79.21);
\definecolor{fillColor}{RGB}{255,0,0}

\path[fill=fillColor,fill opacity=0.50] ( 59.09, 40.68) --
	( 66.68, 37.90) --
	( 74.22, 42.48) --
	( 81.73, 46.39) --
	( 89.26, 51.23) --
	( 96.80, 56.29) --
	(104.33, 61.41) --
	(111.87, 67.28) --
	(119.41, 71.92) --
	(126.95, 77.09) --
	(126.95, 76.97) --
	(119.41, 71.72) --
	(111.87, 66.57) --
	(104.33, 61.34) --
	( 96.80, 55.90) --
	( 89.26, 51.10) --
	( 81.73, 46.34) --
	( 74.22, 41.70) --
	( 66.68, 36.40) --
	( 59.09, 33.43) --
	cycle;

\path[] ( 59.09, 40.68) --
	( 66.68, 37.90) --
	( 74.22, 42.48) --
	( 81.73, 46.39) --
	( 89.26, 51.23) --
	( 96.80, 56.29) --
	(104.33, 61.41) --
	(111.87, 67.28) --
	(119.41, 71.92) --
	(126.95, 77.09);

\path[] (126.95, 76.97) --
	(119.41, 71.72) --
	(111.87, 66.57) --
	(104.33, 61.34) --
	( 96.80, 55.90) --
	( 89.26, 51.10) --
	( 81.73, 46.34) --
	( 74.22, 41.70) --
	( 66.68, 36.40) --
	( 59.09, 33.43);
\definecolor{fillColor}{RGB}{0,0,0}

\path[fill=fillColor,fill opacity=0.50] ( 59.09, 52.91) --
	( 66.68, 53.36) --
	( 74.22, 55.35) --
	( 81.73, 59.06) --
	( 89.26, 63.60) --
	( 96.80, 67.78) --
	(104.33, 72.68) --
	(111.87, 77.96) --
	(119.41, 83.04) --
	(126.95, 88.31) --
	(126.95, 88.25) --
	(119.41, 83.00) --
	(111.87, 77.79) --
	(104.33, 72.68) --
	( 96.80, 67.69) --
	( 89.26, 63.52) --
	( 81.73, 58.76) --
	( 74.22, 55.24) --
	( 66.68, 52.35) --
	( 59.09, 52.38) --
	cycle;

\path[] ( 59.09, 52.91) --
	( 66.68, 53.36) --
	( 74.22, 55.35) --
	( 81.73, 59.06) --
	( 89.26, 63.60) --
	( 96.80, 67.78) --
	(104.33, 72.68) --
	(111.87, 77.96) --
	(119.41, 83.04) --
	(126.95, 88.31);

\path[] (126.95, 88.25) --
	(119.41, 83.00) --
	(111.87, 77.79) --
	(104.33, 72.68) --
	( 96.80, 67.69) --
	( 89.26, 63.52) --
	( 81.73, 58.76) --
	( 74.22, 55.24) --
	( 66.68, 52.35) --
	( 59.09, 52.38);
\definecolor{fillColor}{RGB}{0,191,255}

\path[fill=fillColor,fill opacity=0.50] ( 59.09, 56.28) --
	( 66.68, 53.29) --
	( 74.22, 55.77) --
	( 81.73, 61.62) --
	( 89.26, 71.23) --
	( 96.80, 81.55) --
	( 96.80, 81.53) --
	( 89.26, 71.18) --
	( 81.73, 61.48) --
	( 74.22, 55.02) --
	( 66.68, 52.63) --
	( 59.09, 52.78) --
	cycle;

\path[] ( 59.09, 56.28) --
	( 66.68, 53.29) --
	( 74.22, 55.77) --
	( 81.73, 61.62) --
	( 89.26, 71.23) --
	( 96.80, 81.55);

\path[] ( 96.80, 81.53) --
	( 89.26, 71.18) --
	( 81.73, 61.48) --
	( 74.22, 55.02) --
	( 66.68, 52.63) --
	( 59.09, 52.78);
\definecolor{drawColor}{RGB}{255,0,0}

\path[draw=drawColor,line width= 0.6pt,line join=round] ( 59.09, 35.27) --
	( 66.68, 37.90) --
	( 74.22, 42.19) --
	( 81.73, 46.36) --
	( 89.26, 51.11) --
	( 96.80, 56.27) --
	(104.33, 61.41) --
	(111.87, 66.60) --
	(119.41, 71.89) --
	(126.95, 77.00);
\definecolor{drawColor}{RGB}{0,0,0}

\path[draw=drawColor,line width= 0.6pt,line join=round] ( 59.09, 52.68) --
	( 66.68, 52.94) --
	( 74.22, 55.34) --
	( 81.73, 58.82) --
	( 89.26, 63.60) --
	( 96.80, 67.74) --
	(104.33, 72.68) --
	(111.87, 77.81) --
	(119.41, 83.00) --
	(126.95, 88.30);
\definecolor{drawColor}{RGB}{0,191,255}

\path[draw=drawColor,line width= 0.6pt,line join=round] ( 59.09, 53.33) --
	( 66.68, 52.71) --
	( 74.22, 55.71) --
	( 81.73, 61.52) --
	( 89.26, 71.19) --
	( 96.80, 81.55);
\definecolor{drawColor}{gray}{0.20}

\path[draw=drawColor,line width= 0.6pt,line join=round,line cap=round] ( 46.36, 30.69) rectangle (139.68, 91.06);
\end{scope}
\begin{scope}
\path[clip] (139.68, 30.69) rectangle (232.99, 91.06);
\definecolor{fillColor}{RGB}{255,255,255}

\path[fill=fillColor] (139.68, 30.69) rectangle (232.99, 91.06);
\definecolor{drawColor}{gray}{0.92}

\path[draw=drawColor,line width= 0.3pt,line join=round] (139.68, 31.92) --
	(232.99, 31.92);

\path[draw=drawColor,line width= 0.3pt,line join=round] (139.68, 55.57) --
	(232.99, 55.57);

\path[draw=drawColor,line width= 0.3pt,line join=round] (139.68, 79.21) --
	(232.99, 79.21);

\path[draw=drawColor,line width= 0.3pt,line join=round] (169.37, 30.69) --
	(169.37, 91.06);

\path[draw=drawColor,line width= 0.3pt,line join=round] (203.30, 30.69) --
	(203.30, 91.06);

\path[draw=drawColor,line width= 0.6pt,line join=round] (139.68, 43.74) --
	(232.99, 43.74);

\path[draw=drawColor,line width= 0.6pt,line join=round] (139.68, 67.39) --
	(232.99, 67.39);

\path[draw=drawColor,line width= 0.6pt,line join=round] (139.68, 91.03) --
	(232.99, 91.03);

\path[draw=drawColor,line width= 0.6pt,line join=round] (152.40, 30.69) --
	(152.40, 91.06);

\path[draw=drawColor,line width= 0.6pt,line join=round] (186.33, 30.69) --
	(186.33, 91.06);

\path[draw=drawColor,line width= 0.6pt,line join=round] (220.27, 30.69) --
	(220.27, 91.06);
\definecolor{drawColor}{RGB}{190,190,190}

\path[draw=drawColor,line width= 0.6pt,line join=round] (139.68, 79.21) -- (232.99, 79.21);
\definecolor{drawColor}{gray}{0.50}

\node[text=drawColor,anchor=base,inner sep=0pt, outer sep=0pt, scale=  0.71] at (169.37, 82.15) {1 second};
\definecolor{fillColor}{RGB}{255,0,0}

\path[fill=fillColor,fill opacity=0.50] (152.40, 35.36) --
	(159.99, 38.02) --
	(167.54, 42.79) --
	(175.04, 46.94) --
	(182.57, 51.81) --
	(190.11, 56.26) --
	(197.65, 61.79) --
	(205.19, 66.60) --
	(212.73, 71.88) --
	(220.27, 78.11) --
	(220.27, 77.03) --
	(212.73, 71.72) --
	(205.19, 66.50) --
	(197.65, 61.41) --
	(190.11, 56.03) --
	(182.57, 50.90) --
	(175.04, 45.98) --
	(167.54, 41.73) --
	(159.99, 36.36) --
	(152.40, 34.58) --
	cycle;

\path[] (152.40, 35.36) --
	(159.99, 38.02) --
	(167.54, 42.79) --
	(175.04, 46.94) --
	(182.57, 51.81) --
	(190.11, 56.26) --
	(197.65, 61.79) --
	(205.19, 66.60) --
	(212.73, 71.88) --
	(220.27, 78.11);

\path[] (220.27, 77.03) --
	(212.73, 71.72) --
	(205.19, 66.50) --
	(197.65, 61.41) --
	(190.11, 56.03) --
	(182.57, 50.90) --
	(175.04, 45.98) --
	(167.54, 41.73) --
	(159.99, 36.36) --
	(152.40, 34.58);
\definecolor{fillColor}{RGB}{0,0,0}

\path[fill=fillColor,fill opacity=0.50] (152.40, 52.35) --
	(159.99, 52.52) --
	(167.54, 52.43) --
	(175.04, 53.47) --
	(182.57, 53.89) --
	(190.11, 55.32) --
	(197.65, 58.51) --
	(205.19, 62.40) --
	(212.73, 67.12) --
	(220.27, 74.75) --
	(220.27, 72.60) --
	(212.73, 66.93) --
	(205.19, 62.29) --
	(197.65, 58.36) --
	(190.11, 55.23) --
	(182.57, 53.71) --
	(175.04, 52.67) --
	(167.54, 52.15) --
	(159.99, 52.16) --
	(152.40, 51.94) --
	cycle;

\path[] (152.40, 52.35) --
	(159.99, 52.52) --
	(167.54, 52.43) --
	(175.04, 53.47) --
	(182.57, 53.89) --
	(190.11, 55.32) --
	(197.65, 58.51) --
	(205.19, 62.40) --
	(212.73, 67.12) --
	(220.27, 74.75);

\path[] (220.27, 72.60) --
	(212.73, 66.93) --
	(205.19, 62.29) --
	(197.65, 58.36) --
	(190.11, 55.23) --
	(182.57, 53.71) --
	(175.04, 52.67) --
	(167.54, 52.15) --
	(159.99, 52.16) --
	(152.40, 51.94);
\definecolor{fillColor}{RGB}{0,191,255}

\path[fill=fillColor,fill opacity=0.50] (152.40, 52.80) --
	(159.99, 52.89) --
	(167.54, 55.16) --
	(175.04, 61.66) --
	(182.57, 71.47) --
	(190.11, 81.58) --
	(190.11, 81.51) --
	(182.57, 71.17) --
	(175.04, 61.49) --
	(167.54, 54.98) --
	(159.99, 52.57) --
	(152.40, 52.12) --
	cycle;

\path[] (152.40, 52.80) --
	(159.99, 52.89) --
	(167.54, 55.16) --
	(175.04, 61.66) --
	(182.57, 71.47) --
	(190.11, 81.58);

\path[] (190.11, 81.51) --
	(182.57, 71.17) --
	(175.04, 61.49) --
	(167.54, 54.98) --
	(159.99, 52.57) --
	(152.40, 52.12);
\definecolor{drawColor}{RGB}{255,0,0}

\path[draw=drawColor,line width= 0.6pt,line join=round] (152.40, 34.64) --
	(159.99, 37.79) --
	(167.54, 41.78) --
	(175.04, 46.38) --
	(182.57, 51.61) --
	(190.11, 56.12) --
	(197.65, 61.50) --
	(205.19, 66.52) --
	(212.73, 71.82) --
	(220.27, 77.03);
\definecolor{drawColor}{RGB}{0,0,0}

\path[draw=drawColor,line width= 0.6pt,line join=round] (152.40, 52.08) --
	(159.99, 52.30) --
	(167.54, 52.33) --
	(175.04, 52.95) --
	(182.57, 53.88) --
	(190.11, 55.27) --
	(197.65, 58.48) --
	(205.19, 62.34) --
	(212.73, 67.06) --
	(220.27, 74.74);
\definecolor{drawColor}{RGB}{0,191,255}

\path[draw=drawColor,line width= 0.6pt,line join=round] (152.40, 52.25) --
	(159.99, 52.64) --
	(167.54, 55.03) --
	(175.04, 61.61) --
	(182.57, 71.19) --
	(190.11, 81.58);
\definecolor{drawColor}{gray}{0.20}

\path[draw=drawColor,line width= 0.6pt,line join=round,line cap=round] (139.68, 30.69) rectangle (232.99, 91.06);
\end{scope}
\begin{scope}
\path[clip] ( 46.36, 91.06) rectangle (139.68,107.63);
\definecolor{drawColor}{gray}{0.20}
\definecolor{fillColor}{gray}{0.85}

\path[draw=drawColor,line width= 0.6pt,line join=round,line cap=round,fill=fillColor] ( 46.36, 91.06) rectangle (139.68,107.63);
\definecolor{drawColor}{gray}{0.10}

\node[text=drawColor,anchor=base,inner sep=0pt, outer sep=0pt, scale=  0.73] at ( 93.02, 96.31) {label.density: 0.001};
\end{scope}
\begin{scope}
\path[clip] (139.68, 91.06) rectangle (232.99,107.63);
\definecolor{drawColor}{gray}{0.20}
\definecolor{fillColor}{gray}{0.85}

\path[draw=drawColor,line width= 0.6pt,line join=round,line cap=round,fill=fillColor] (139.68, 91.06) rectangle (232.99,107.63);
\definecolor{drawColor}{gray}{0.10}

\node[text=drawColor,anchor=base,inner sep=0pt, outer sep=0pt, scale=  0.73] at (186.33, 96.31) {label.density: 0.1};
\end{scope}
\begin{scope}
\path[clip] (  0.00,  0.00) rectangle (238.49,144.54);
\definecolor{drawColor}{gray}{0.20}

\path[draw=drawColor,line width= 0.6pt,line join=round] ( 59.09, 27.94) --
	( 59.09, 30.69);

\path[draw=drawColor,line width= 0.6pt,line join=round] ( 93.02, 27.94) --
	( 93.02, 30.69);

\path[draw=drawColor,line width= 0.6pt,line join=round] (126.95, 27.94) --
	(126.95, 30.69);
\end{scope}
\begin{scope}
\path[clip] (  0.00,  0.00) rectangle (238.49,144.54);
\definecolor{drawColor}{gray}{0.30}

\node[text=drawColor,anchor=base,inner sep=0pt, outer sep=0pt, scale=  0.73] at ( 59.09, 19.68) {1e+01};

\node[text=drawColor,anchor=base,inner sep=0pt, outer sep=0pt, scale=  0.73] at ( 93.02, 19.68) {1e+03};

\node[text=drawColor,anchor=base,inner sep=0pt, outer sep=0pt, scale=  0.73] at (126.95, 19.68) {1e+05};
\end{scope}
\begin{scope}
\path[clip] (  0.00,  0.00) rectangle (238.49,144.54);
\definecolor{drawColor}{gray}{0.20}

\path[draw=drawColor,line width= 0.6pt,line join=round] (152.40, 27.94) --
	(152.40, 30.69);

\path[draw=drawColor,line width= 0.6pt,line join=round] (186.33, 27.94) --
	(186.33, 30.69);

\path[draw=drawColor,line width= 0.6pt,line join=round] (220.27, 27.94) --
	(220.27, 30.69);
\end{scope}
\begin{scope}
\path[clip] (  0.00,  0.00) rectangle (238.49,144.54);
\definecolor{drawColor}{gray}{0.30}

\node[text=drawColor,anchor=base,inner sep=0pt, outer sep=0pt, scale=  0.73] at (152.40, 19.68) {1e+01};

\node[text=drawColor,anchor=base,inner sep=0pt, outer sep=0pt, scale=  0.73] at (186.33, 19.68) {1e+03};

\node[text=drawColor,anchor=base,inner sep=0pt, outer sep=0pt, scale=  0.73] at (220.27, 19.68) {1e+05};
\end{scope}
\begin{scope}
\path[clip] (  0.00,  0.00) rectangle (238.49,144.54);
\definecolor{drawColor}{gray}{0.30}

\node[text=drawColor,anchor=base east,inner sep=0pt, outer sep=0pt, scale=  0.73] at ( 41.41, 40.71) {1e-03};

\node[text=drawColor,anchor=base east,inner sep=0pt, outer sep=0pt, scale=  0.73] at ( 41.41, 64.36) {1e-01};

\node[text=drawColor,anchor=base east,inner sep=0pt, outer sep=0pt, scale=  0.73] at ( 41.41, 88.00) {1e+01};
\end{scope}
\begin{scope}
\path[clip] (  0.00,  0.00) rectangle (238.49,144.54);
\definecolor{drawColor}{gray}{0.20}

\path[draw=drawColor,line width= 0.6pt,line join=round] ( 43.61, 43.74) --
	( 46.36, 43.74);

\path[draw=drawColor,line width= 0.6pt,line join=round] ( 43.61, 67.39) --
	( 46.36, 67.39);

\path[draw=drawColor,line width= 0.6pt,line join=round] ( 43.61, 91.03) --
	( 46.36, 91.03);
\end{scope}
\begin{scope}
\path[clip] (  0.00,  0.00) rectangle (238.49,144.54);
\definecolor{drawColor}{RGB}{0,0,0}

\node[text=drawColor,anchor=base,inner sep=0pt, outer sep=0pt, scale=  0.92] at (139.68,  7.64) {Number of data $N$};
\end{scope}
\begin{scope}
\path[clip] (  0.00,  0.00) rectangle (238.49,144.54);
\definecolor{drawColor}{RGB}{0,0,0}

\node[text=drawColor,rotate= 90.00,anchor=base,inner sep=0pt, outer sep=0pt, scale=  0.92] at ( 13.08, 60.87) {seconds};
\end{scope}
\begin{scope}
\path[clip] (  0.00,  0.00) rectangle (238.49,144.54);
\definecolor{drawColor}{RGB}{0,0,0}

\node[text=drawColor,anchor=base west,inner sep=0pt, outer sep=0pt, scale=  1.10] at ( 46.36,129.95) {LOPART is $O(N\log N)$ time};

\node[text=drawColor,anchor=base west,inner sep=0pt, outer sep=0pt, scale=  1.10] at ( 46.36,115.69) {with $O(N)$ positive labels};
\end{scope}
\end{tikzpicture}